\documentclass[nohyperref]{article}

\usepackage{microtype}
\usepackage{graphicx}
\usepackage{subfigure}
\usepackage{booktabs} 

\usepackage{hyperref}

\usepackage[accepted]{icml2022}

\usepackage{amsmath}
\usepackage{amssymb}
\usepackage{mathtools}
\usepackage{amsthm}

\usepackage[capitalize,noabbrev]{cleveref}

\usepackage{amsmath}
\usepackage{amssymb}
\usepackage{amsthm}
\usepackage{graphicx}
\usepackage{enumerate}
\usepackage{microtype}
\usepackage{mleftright}
\usepackage{mdframed}
\usepackage{authblk}
\usepackage{tcolorbox}
\usepackage[english]{babel}
\usepackage[utf8]{inputenc}
\usepackage{xspace}

\makeatletter
\renewcommand{\ALG@name}{Protocol}
\makeatother
\allowdisplaybreaks

\graphicspath{ {./figures/} }

\theoremstyle{plain}
\newtheorem{theorem}{Theorem}[section]

\newtheorem{lemma}[theorem]{Lemma}

\theoremstyle{definition}
\newtheorem{definition}[theorem]{Definition}

\theoremstyle{remark}

\newtheorem{observation}[theorem]{Observation}
\newtheorem*{theorem*}{Theorem}

\providecommand{\poly}{{\mathsf{poly}}}
\providecommand{\OPT}{{\mathit{OPT}}}
\DeclareMathOperator{\DISJ}{DISJ}
\DeclareMathOperator{\polylog}{\mathsf{polylog}}

\newcommand{\BoostAttempt}{\textsf{BoostAttempt}\xspace}
\newcommand{\AccuratelyClassify}{\textsf{AccuratelyClassify}\xspace}
\newcommand{\AdaBoost}{AdaBoost\xspace}

\usepackage[textsize=tiny]{todonotes}

\icmltitlerunning{A Resilient Distributed Boosting Algorithm}

\begin{document}

\twocolumn[
\icmltitle{A Resilient Distributed Boosting Algorithm}



\icmlsetsymbol{equal}{*}

\begin{icmlauthorlist}
\icmlauthor{Yuval Filmus}{equal,TeCS}
\icmlauthor{Idan Mehalel}{equal,TeCS}
\icmlauthor{Shay Moran}{equal,TeM,G}
\end{icmlauthorlist}

\icmlaffiliation{TeCS}{The Henry and Marilyn Taub Faculty of Computer Science, Technion, Haifa, Israel}
\icmlaffiliation{TeM}{Faculty of Mathematics, Technion, Haifa, Israel}
\icmlaffiliation{G}{Google Research, Israel}

\icmlcorrespondingauthor{Idan Mehalel}{idanmehalel@cs.technion.ac.il}

\icmlkeywords{Machine Learning, ICML}

\vskip 0.3in
]



\printAffiliationsAndNotice{\icmlEqualContribution} 

\begin{abstract}
Given a learning task where the data is distributed among several parties, communication is one of the fundamental resources which the parties would like to minimize.
We present a distributed boosting algorithm which is resilient to a limited amount of noise. Our algorithm is similar to classical boosting algorithms, although it is equipped with a new component, inspired by Impagliazzo's hard-core lemma \cite{impagliazzo1995hard}, adding a robustness quality to the algorithm. We also complement this result by showing that resilience to any asymptotically larger noise is not achievable by a communication-efficient algorithm.




\end{abstract}

\section{Introduction}


Most work in learning theory focuses on designing efficient learning algorithms which generalize well. New considerations arise when the training data is spread among several parties: speech recorded on different smartphones, medical data gathered from several clinics, and so on. In such settings, it is important to minimize not only the computational complexity, but also the communication complexity. Apart from practical considerations of limited bandwidth, minimizing the communication complexity also limits the amount of data being exposed to prying ears. This motivates designing distributed learning algorithms, which improve on the naive idea of sending all training data to a single party. 


In the classical PAC model, distributed learning has been studied mostly in the realizable setting, where it was shown that distributed implementations of boosting algorithms 
can learn any VC class with communication complexity which is polynomial in the description length (in bits) of a single example~\citep*{balcan2012distributed,daume2012efficient,kane2019communication}.

\smallskip

In this work we deviate from the realizable setting and allow a small amount of noise in the input sample. In our setting there are $k$ players and a center, who are given a domain~$U$ of size $\lvert U\rvert = n$ and a concept class $\mathcal{H}$ over $U$ with VC dimension $d \ll n$. For a labelled input sample~$S$ distributed among the players, let $\OPT:=\OPT(S)\in\mathbb{N}$ denote the number of examples in~$S$ which are misclassified by the best hypothesis in $\mathcal{H}$. In most parts of the paper, we require that~$\OPT \in \polylog n$.
The goal of the parties is to learn together a classifier $f$ that has at most~$\OPT$ errors on~$S$, while using $\poly(d, k, \log |S|, \log n)$ bits of communication. Note that~$\log n$ is the number of bits needed to encode a single point in $U$, and thus $\polylog n$ means polynomial in the description length of a single example.\footnote{We refer the reader to \citep*{kane2019communication,braverman2019convex} for a more thorough discussion regarding the choice of $\polylog n$ as a ``yardstick'' for communication efficiency.}

\paragraph{Main result.}
Our main result, formally stated in Theorems \ref{theo:upp_acc_int} and \ref{theo:low_acc_opt_int} asserts the following:
for every VC class, if the minimal error of an hypothesis satisfies $\OPT\in \polylog n$, 
then a simple robust variant of classical boosting learns it with $\poly(d, k, \log |S|, \log n)$ communication complexity.
Conversely, when $\OPT \notin \polylog n$, there exist one-dimensional VC classes for which \emph{any} learning algorithm has super-polylogarithmic communication complexity.

The novelty of our algorithm lies in a non-standard usage of boosting that identifies small ``hard'' sets for which any hypothesis from the class has large error. This kind of usage resembles (and is inspired by) Impagliazzo's hard-core lemma \citep*{impagliazzo1995hard}, in particular its proof using the method of multiplicative weights. Our negative result is a slight extension of the argument appearing in~\citep*{kane2019communication}.

\smallskip

We note that our positive result can alternatively be obtained by a reduction to \emph{semi-agnostic learning} \citep*{bun2019private}, that is, agreeing on a classifier with at most $c \cdot \OPT$ errors for some constant~$c$. Semi-agnostic learning is possible using $\poly(d, k, \log |S|, \log n)$ bits of communication by the works of \citet*{balcan2012distributed, chen2016communication}. 
Given a semi-agnostic communication protocol with a constant approximation factor $c$ and communication complexity $\poly(d, k, \log |S|, \log n)$, one can proceed as follows: execute the semi-agnostic protocol to obtain a hypothesis $f$, and have each player broadcast her examples that $f$ misclassifies. Then, the players modify $f$ on the misclassified points and output an optimal hypothesis $f'$. If there exists an hypothesis in the class whose error is $\polylog(n)$ then the communication cost of this step is $\poly(d, k, \log |S|, \log n)$, and thus the overall communication complexity is $\poly(d, k, \log |S|, \log n)$.

The advantage of our approach is the simplicity of our protocol, which is a simple modification of the classical boosting approach that makes it resilient to mild noise.
This is in contrast with semi-agnostic learning protocol which rely on non-trivial subroutines
(e.g.\ the distributed implementation of Bregman projection in the protocol by \citet*{chen2016communication}).

\paragraph{Empirical loss versus population loss.}
From a technical perspective, this work focuses on 
\emph{distributed empirical risk minimization with efficient communication complexity}; that is, the objective is to design an efficient distributed protocol which minimizes the empirical loss.

While this deviates from the main objective in statistical learning of minimizing the \emph{population loss}, we focus on the empirical loss for the following reasons:
\begin{enumerate}[(i)]
\item \emph{Efficient communication implies generalization}: 
As discussed in~\citep*{kane2019communication,braverman2019convex},
Occam's razor and sample compression arguments can be naturally used
to bound the \emph{generalization gap} --- i.e.\ the absolute difference between the empirical and population losses --- of efficient distributed learning algorithms.
In a nutshell, the bound follows by arguing that the
output hypothesis is determined by the communication transcript of the protocol.
Hence, the communication complexity of the protocol upper-bounds the \emph{description length} of the output hypothesis, which translates to a bound on the generalization
gap via Occam's razor or sample compression.
In particular, this reasoning applies to the algorithm we present in this work, and hence it generalizes. 
Thus, for communication-efficient protocols, the empirical loss is a good proxy of the population loss.
\item Focusing on empirical loss simplifies the exposition: while it is possible to translate our results to 
the setting of population loss, this introduces additional probabilistic machinery and complicates the presentation without introducing any new ideas. Further, Empirical risk minimization is a natural and classical problem, 
and previous work on distributed PAC learning focused on it, at least implicitly~\citep*{kane2019communication,vempala2020communication,braverman2019convex}.
\end{enumerate}

\paragraph{Paper organization.} In Section~\ref{sec:mod_res} we formally define the model and give an overview of our results and related work. Section~\ref{sec:prel} contains brief preliminaries.
We prove the upper bound in Section~\ref{sec:ub-exact} and the lower bound in Section~\ref{sec:lb-exact}. The paper closes with Section~\ref{sec:open}, which discusses directions for future research.

\section{Model and results} \label{sec:mod_res}

\subsection{Model}
Following~\citep*{balcan2012distributed}, 
    we consider a distributed setting consisting of~$k$ players 
    numbered $1,\dots,k$, and a center.
    Each player can communicate only with the center.  
    An hypothesis class $\mathcal{H}$ over a universe $\mathcal{U}$ is given, 
    and a finite domain set $U\subset \mathcal{U}$ of size $n$ is given as well.
    We denote the VC-dimension of~$\mathcal{H}$ by $d:=d(\mathcal{H})$.
    The finite domain $U$ is known in advance to the center and to all players. 
    A pair $z:=(x,y)$, where $x \in U$ and $y \in \{\pm 1\}$, is called an \emph{example}. 
    A sequence of examples $z_1, \dots, z_m$ is called a \emph{sample}, 
    and denoted by $S$. 
    For a classifier $f\colon \mathcal{U}\to\{\pm 1\}$, 
    let~$E_S(f)$ denote the number of examples in $S$ that $f$ misclassifies:
    \[ E_S(f):=\sum_{(x,y)\in S}1[f(x)\neq y]. \]
    Let $\OPT$ be the number of misclassified
    examples in $S$ with respect to the best hypothesis in $\mathcal{H}$:
\[
\OPT=\OPT(S,\mathcal{H}):=\min_{h\in \mathcal{H}}E_S(h).
\]

In most parts of the paper we require that $\OPT \in \polylog n$. In our setting, a sample $S$ is \emph{adversarially distributed} between the $k$ players into $k$ subsamples $S_1, \dots , S_k$. Note that the center gets no input.
We use the notation $S=\langle S_i \rangle_{i=1}^k$ 
    to clarify that player~$i$ has a fraction $S_i$ of the sample, 
    and concatenating all the $S_i$'s yields the entire input sample~$S$. 
   
The goal is to \emph{learn} $\mathcal{H}$, which we define as follows:

\begin{definition} \label{def:agn_lear}
Let $\mathcal{H}$ be a concept class over a (possibly infinite) universe $\mathcal{U}$ and let $k$ denote the number of players. 
For a function $T\colon\mathbb{N}\rightarrow \mathbb{N}$ we say that
$\mathcal{H}$ is \emph{learnable under the promise $\OPT \leq T(n)$}
if there exists a communication complexity bound $C(d,k,n,m)\in\poly(d,k,\log n,\log m)$ such that for every finite $U\subseteq\mathcal{U}$ of size $n=\lvert U\rvert$, 
there exists a distributed algorithm~$\pi=\pi(U)$ that satisfies the following.
For every input sample $S = \langle S_i \rangle_{i=1}^k$ with $m$ examples from $U$, if $\OPT=\OPT(S,\mathcal{H})\leq T(n)$ then the $k$ parties and the center
agree on an output hypothesis $f$ which satisfies $E_S(f) \leq \OPT$ with probability at least $\frac{2}{3}$ (over the randomness of the protocol $\pi$, when randomized), while transmitting at most $C(d,k,n,m)$ bits.
\end{definition}

Let us make a few remarks in order to clarify some choices made in the above definition.
\begin{enumerate}
    
    
    \item \textbf{Infinite classes}.
    The above definition allows one to handle natural infinite classes $\mathcal{H}$ such as Euclidean halfspaces. The finite subdomain $U\subseteq \mathcal{U}$ models a particular instance of the learning task defined by $\mathcal{H}$. For example, if $\mathcal{H}$ is the class of halfspaces in $\mathbb{R}^d$, and we use an encoding of real numbers with $B$ bits, then $U$ consists of all possible $2^{d\cdot B}$ points in $\mathbb{R}^d$ that can be encoded. The universal quantification over $U$ serves to make the definition scalable and independent of the encoding of the input points.
    \item \textbf{The protocol may depend on $\boldsymbol{U}\boldsymbol{\subseteq}\boldsymbol{\mathcal{U}}$.} This possible dependence reflects the fact that when designing algorithms in practice,
    one knows how the domain points are being encoded as inputs.\footnote{We remark however that the protocol appearing in Section~\ref{sec:ub-exact} is uniform in~$U$, that is, it can accept~$U$ as an additional input.}
\end{enumerate}

\subsection{Results}

Our positive result is stated in the following theorem.

\begin{theorem}[Positive Result] \label{theo:upp_acc_int}
Let $\mathcal{H}$ be a concept class with $d(\mathcal{H})<\infty$ and let $T=T(n)\in \polylog n$. 
Then, $\mathcal{H}$ is learnable under the promise $\OPT \leq T(n)$, and this is achieved by a simple variant of classical boosting.
Furthermore, the algorithm is deterministic and oblivious to $T$ and $\OPT$.
\end{theorem}


The protocol we use to prove Theorem~\ref{theo:upp_acc_int} is a \emph{resilient} version of realizable-case boosting. It is resilient in the sense that it can be applied to any input sample, including samples that are \emph{not} realizable by the class $\mathcal{H}$. 
Moreover, as long as the input sample is sufficiently close to being realizable, this variant of boosting enjoys similar guarantees as in the fully realizable case.
This feature of our protocol is not standard in boosting algorithms in the realizable case, which are typically vulnerable to noise \citep*{dietterich2000experimental, long2010random}.

Our protocol can be implemented in the no-center model, in which the players can communicate directly (see \citep*{balcan2012distributed} for a more thorough discussion of these two models), by having one of the players play the part of the center. It also admits a randomized computationally efficient implementation, assuming an oracle access to a PAC learning algorithm for~$\mathcal{H}$ in the centralized setting
(see Section~\ref{sec:ub-exact} for further discussion). On the other hand, the protocol is improper. This is unavoidable: a result by~\citet*{kane2019communication} shows that even in the realizable case (i.e.\ $T(n)=0$), 
    some VC classes cannot be properly learned by communication-efficient protocols.



As mentioned in the introduction, the positive result of Theorem~\ref{theo:upp_acc_int} can also be proved by reduction to semi-agnostic learning. However, our direct approach results in a simpler protocol.

\smallskip

The following negative result shows that the assumption $\OPT \in \polylog n$ made by Theorem~\ref{theo:upp_acc_int} is necessary for allowing communication-efficient learning, even if the protocol is allowed to be randomized and improper.

\begin{theorem}[Negative Result] \label{theo:low_acc_opt_int}
Let $\mathcal{H}=\{h_n: n\in \mathbb{N}\}$, where $h_n(i)=1$ if and only if $i=n$, be the class of singletons over $\mathbb{N}$.
If $T(n) = \log^{\omega(1)} (n)$ then $\mathcal{H}$ is not learnable under the promise that $\OPT\leq T(n)$, even when there are only $k=2$ players.
\end{theorem}

When there are two players, our model is equivalent to the standard two-party communication model~\citep*{yao1979some, kushilevitz_nisan_1996, rao_yehudayoff_2020}, in which two players, Alice and Bob, communicate through a direct channel, and this is the setting in which we prove Theorem~\ref{theo:low_acc_opt_int}.

\smallskip

Our results are in fact more general than stated. The algorithm used to prove Theorem~\ref{theo:upp_acc_int} outputs a hypothesis making at most $\OPT$ many mistakes using $\OPT \cdot \poly(d,k, \log m,\log n)$ communication (without having to know $\OPT$ in advance). The lower bound used to prove Theorem~\ref{theo:low_acc_opt_int} shows that for any value $T(n)$ and for any algorithm that learns the class of singletons there exists an input sample with $\OPT \approx T(n)$ on which the communication complexity of the protocol is $\Omega(T(n))$.




\subsection{Related Work}

Originally, distributed learning was studied from the point of view of parallel computation (a partial list includes~\citep*{bshouty1997exact, collins2002logistic, zinkevich2010parallelized, long2011algorithms}). The focus was on reducing the time complexity rather than the communication complexity. More recent work aims at minimizing communication~\citep*{balcan2012distributed, daume2012protocols, daume2012efficient, blum2021communication}. In~\citep*{balcan2012distributed}, privacy aspects of such learning tasks are discussed as well. 

A related natural model of distributed learning was proposed in~\citep*{balcan2012distributed}. In this model, there are $k$ entities and a center, and each entity $i$ can draw examples from a distribution~$D_i$. The goal is to learn a good hypothesis with respect to the mixture distribution $D=\frac{1}{k}\sum_{i=1}^k D_i$. The communication topology in this model is a star: all entities can communicate only with the center.

In this work, we consider a slightly different model, studied by~\citet*{daume2012efficient, daume2012protocols, kane2019communication, braverman2019convex}, which we call the \emph{adversarial} model.
In this model, a sample $S$ is given and partitioned freely among $k$ players by an adversary. While this model might seem less natural, it is more general than the model of \citet*{balcan2012distributed}, and our main contribution is a protocol that can be applied to this general model.


The work by \citet*{lazarevic2001distributed} suggested a framework for using boosting in distributed environments.
In~\citep*{chen2016communication},  a clever analysis of ``Smooth Boosting" \citep*{kale2007boosting} is used to give an efficient semi-agnostic boosting protocol. \citet*{kane2019communication} characterize which classes can be learned in the distributed and proper setting, and give some bounds for different distributed learning tasks. In~\citep*{braverman2019convex}, tight lower and upper bounds on the communication complexity of learning halfspaces are given, using geometric tools.




\section{Preliminaries} \label{sec:prel}

We use $\log$ for the base~$2$ logarithm. Let $U$ be the domain set and let $S\subset U \times \{\pm 1\}$ be a sample. 
We follow the standard definitions of \emph{empirical loss} and \emph{loss}:
\begin{align*}
    L_S(f) & := \frac{1}{\lvert S \rvert} \sum_{(x,y)\in S} 1\mleft[f(x) \neq y\mright], \\ 
    L_p(f) & :=\Pr_{(x,y)\sim p} \mleft[f(x) \neq y\mright],
\end{align*}
respectively, where $p$ is a probability distribution over $U\times \{\pm 1\}$.
We now briefly overview some relevant technical tools.

\paragraph{Boosting.}
The seminal work of \citet*{freund1997decision} used the \AdaBoost algorithm to boost a ``weak" learner into a ``strong" one. 
In this work we use a simplified version of \AdaBoost (see \citep*{schapire2013boosting}):
given a distribution $p$ over a sample $S$, an \emph{$\alpha$-weak hypothesis} with respect to $p$ is a hypothesis $h$ which is better than a random guess by an additive factor of $\alpha$:
\[
\Pr_{(x,y)\sim p}\mleft[h(x)\neq y\mright] \leq 1/2-\alpha.
\]
The boosting algorithm requires an oracle access to an \emph{$\alpha$-weak learner}, which is an algorithm that returns $\alpha$-weak hypotheses.
Given such a weak learner, the boosting algorithm operates as follows: it receives as input a sample $S$, and initializes the weight $W_1\mleft(z:=(x,y)\mright)$ of any example $z\in S$ to be $1$. In each iteration $t\in {1,\dots, T}$, it then uses the weak learner to obtain an $\alpha$-weak hypothesis~$h_t$ with respect to the distribution $p_t$ on $S$, 
which is defined by the weight function $W_t$, i.e.\ the probability of each example $z$ is proportional to $W_t(Z)$. 
The weights are then updated according to the performance of the weak hypothesis $h_t$ on each example:
\[
W_{t+1}(z)=W_{t}\cdot 2^{-1 [h(x) = y]}.
\]
After $T$ iterations, the algorithm returns the classifier
\[
f=\operatorname{sign} \mleft(\sum _{t=1}^T h_t\mright).
\]
We have the following upper bound\footnote{This formulation of the theorem appears explicitly as Lemma~2 in \citep*{kane2019communication}.} on the value of $T$ required for $f$ to satisfy $E_S(f)=0$.

\begin{theorem} [{\citet*{freund1997decision}}] 
\label{theo:boosting}
Let $T \geq 6\log \lvert S\rvert$, and assume that in any iteration $t$, a hypothesis $h_t$ which is $\mleft(\frac{1}{2}-\frac{1}{15}\mright)$-weak with respect to the current distribution $p_t$ is provided to the variant of \AdaBoost described above. Then for any $(x,y)\in S$ we have
\[
\frac{1}{T} \sum _{t=1}^T 1\mleft[h_t(x) \neq y\mright] \leq 1/3.
\]
\end{theorem}
An immediate corollary is that if $f$ is the classifier returned by \AdaBoost and $T \geq 6\log \lvert S\rvert$, then $E_S(f)=0$.

\paragraph{Small $\boldsymbol\epsilon$-approximations.}
Let $\mathcal{H}\subseteq \{\pm 1\}^{\mathcal{U}}$ be a concept class of VC-dimension $d< \infty$, let $p$ be a distribution over examples in $\mathcal{U}\times\{\pm 1\}$, and let $\epsilon>0$. The seminal uniform convergence theorem of \citet*{vapnik71uniform} implies that a random i.i.d sample $S$ of size $|S|=O(d/\epsilon^2)$  which is drawn from $p$ satisfies with a positive probability that
\[
(\forall h\in \mathcal{H}): \mleft|L_S(h)-L_p(h)\mright| \leq \epsilon.
\]
Crucially, note that $|S|$ depends only on $d,\epsilon$.
In particular, for every distribution $p$ there exists such a sample in its support.

\paragraph{Communication complexity.}
Our negative result applies already when there are only two players, in which case our model is equivalent to the standard two-party communication model \citep*{yao1979some,kushilevitz_nisan_1996}.
One of the standard problems in the two-party communication model is \emph{set disjointness}. In this problem, Alice gets a string $x \in \{0,1\}^n$, Bob gets a string $y \in \{0,1\}^n$, and the goal is to compute the following function $\DISJ_n(x,y)$:
\[
\DISJ_n(x,y)
=
\begin{cases}
0 & x_i=y_i=1 \text{ for some } $i$, \\
1 & \text{otherwise}.
\end{cases}
\]

The randomized communication complexity of $\DISJ_n$ is known to be large:

\begin{theorem}[\citet*{razborov1990distributional,kalyanasundaram1992probabilistic}] \label{theo:set_disj}
The randomized communication complexity of $\DISJ_n$ is $\Theta(n)$.
\end{theorem}

\section{A resilient boosting protocol} \label{sec:ub-exact}

In this section we use our boosting variant to prove Theorem~\ref{theo:upp_acc_int}, which follows from the next theorem.  

\begin{theorem} \label{theo:upp_acc}
{Let $\mathcal{H}$ be an hypothesis class with VC dimension $d<\infty$, let $k$ be the number of players, and let $T(n)\in \polylog(n)$.} The protocol \AccuratelyClassify, described in Figure~\ref{fig:prot_upp_imp_main}, is a learning protocol under the promise $\OPT \leq T(n)$, and has communication complexity of
\[O\bigl(\OPT \cdot k\log \lvert S \rvert(d \log n+ \log \lvert S \rvert)\bigr),\]  
where $S$ is its input sample. Furthermore, if $S$ contains no contradicting examples (that is, examples $(x,+1),(x,-1)$) then the classifier $f$ which the protocol outputs is consistent (i.e.\ satisfies $E_S(f)=0$).
\end{theorem}

\begin{figure}
    \centering
    \begin{tcolorbox}
    \begin{center}
        \BoostAttempt: Boosting that may get ``stuck"
    \end{center}
    \textbf{Setting:} There are $k$ players and a center, and $\mathcal{H}$ is a known hypothesis class over a domain~$U$. \\
    \textbf{Input:} A distributed sample $S:=\langle S_i \rangle_{i=1}^k$, where $S_i= (x^i_1,y^i_1), \dots, (x^i_{\lvert S_i \rvert},y^i_{\lvert S_i \rvert})$ for $i\in [k]$.\\
    \textbf{Output:} Either all players agree on a classifier $f\colon U \rightarrow \{\pm 1\}$ which makes no errors on $S$,
    or each player $i$ holds a sample $S'_i\subseteq S_i$ such that the concatenated sample $S'=\langle S'_i\rangle_{i=1}^k$ is not realizable.
    The center holds $S'$.
    \begin{enumerate}
        \item \textbf{Initialize:} Each player $i$ initializes $W_1(z^i_j)=1$ for all $1\leq j \leq \lvert S_i \rvert$.
        \item For $t:=1,\dots, T=\lceil 6\log \lvert S\rvert \rceil$:
        \begin{enumerate}
            \item For all $i\in [k]$, let $p_t^i$ be the distribution over $S_i$ defined by 
            $p_t^i(z^i_j)=\frac{W_t(z^i_j)}{W_t^{(i)}}$,
            where $W_t^{(i)}= \sum_{1\leq j\leq \lvert S_i \rvert} W_t(z^i_j)$. \\
            Each player~$i$ sends to the center a $\frac{1}{100}$-approximation w.r.t.\ $p_t^i$ of minimal size, 
            denoted by $S'_i= \hat{z}^i_1, \dots, \hat{z}^i_{\lvert S'_i \rvert}$.
            \item Each player $i$ sends $W_t^{(i)}$ to the center.
            \item Let $S' = \langle S'_i \rangle_{i=1}^k$. Let $D_t$ be the distribution on $S'$ defined by
            $D_t(\hat{z}^i_j) = \frac{1}{\lvert S'_i \rvert} \cdot \frac{W_t^{(i)}}{W_t}$,
            where $W_t= \sum_{i=1}^k W_t^{(i)}$ is the total sum of weights.
            \item If there is $\hat{h}\in \mathcal{H}$ such that $L_{D_t}(\hat{h}) \leq 1/100$ then:
            \begin{itemize}
                \item The center sets $h_t:=\hat{h}$ and sends $h_t$ to all players.
            \end{itemize}
            \item Else:
            \begin{itemize}
                \item Output $S'$.
            \end{itemize}
            \item Each player $i$ updates
            \[
            W_{t+1}(z^i_j) = W_t(z^i_j)\cdot 2^{-1[h_t(x^i_j)=y^i_j]}
            \]
            for any $z^i_j\in S_i$.
        \end{enumerate}
        \item Output the classifier
        \[
        f(x)=\operatorname{sign}\mleft(\sum_{t=1}^Th_t(x)\mright).
        \]
    \end{enumerate}
    \end{tcolorbox}
    \caption{A boosting protocol that may get ``stuck" when the input sample is not realizable.}
    \label{fig:prot_upp_imp_boost}
\end{figure}

\begin{figure}
    \centering
    \begin{tcolorbox}
    \begin{center}
        \AccuratelyClassify: A learning protocol
    \end{center}
    \textbf{Setting:} There are $k$ players and a center, and $\mathcal{H}$ is a known hypothesis class over a domain~$U$.\\
    \textbf{Input:} A distributed sample $S:=\langle S_i \rangle_{i=1}^k$. (Below we treat each $S_i$ as a multiset.)\\
    \textbf{Output:} A classifier $f\colon U \rightarrow \{\pm 1\}$.
    \begin{enumerate}
        \item \textbf{Initialize:} The center initializes a multiset $\mathcal{D}:=\emptyset$.
        \item While $\operatorname{\BoostAttempt}(\langle S_i\rangle _{i=1}^k)$ returns a non-realizable subsample $S':= \langle S'_i \rangle_{i=1}^k$:
        \begin{enumerate}
            \item The center updates $\mathcal{D} := \mathcal{D} \cup S'$.
            \item Each player updates $S_i:=S_i\backslash S'_i$.
        \end{enumerate}
        \item Let $g$ be the classifier returned by \BoostAttempt. 
        \item For every $x \in U$, let $n_+(x)$ be the number of times that the example $(x,+1)$ occurs in $\mathcal{D}$, and define $n_-(x)$ similarly.
        \item Output the classifier $f\colon U\rightarrow \{\pm 1\}$ defined for any $x\in U$ as follows:
        \[
        f(x)=
        \begin{cases}
            +1 & \text{} n_+(x) \geq 1 \text{, } n_+(x) \geq n_-(x), \\
            -1 & \text{} n_-(x) \geq 1 \text{, } n_-(x) > n_+(x), \\
            g(x) & \text{otherwise}.
        \end{cases}
        \]
    \end{enumerate}
    \end{tcolorbox}
    \caption{A resilient improper, deterministic learning protocol.}
    \label{fig:prot_upp_imp_main}
\end{figure}

\AccuratelyClassify relies on the \BoostAttempt protocol, appearing in Figure~\ref{fig:prot_upp_imp_boost}, which is similar to classical boosting.
 
To prove the theorem, we first argue that if \BoostAttempt does not get stuck (i.e.\ it reaches Item~3 in Figure~\ref{fig:prot_upp_imp_boost}), then it simulates boosting and enjoys the guarantees stated in Theorem~\ref{theo:boosting}. Then, we take into account what happens when \BoostAttempt does get stuck;
in this case we adopt the perspective inspired by Impagliazzo's Hardcore lemma to remove a small subsample of the input which is ``hard'' in the sense 
that every hypothesis in $\mathcal{H}$ has large error on it.
Finally, we analyze the total communication cost of the two protocols. 

\begin{lemma} \label{lem:upp_acc_zero_err}
 If protocol \BoostAttempt, described in Figure~\ref{fig:prot_upp_imp_boost}, outputs a classifier $f$, then $E_S(f)=0$.
\end{lemma}

\begin{proof}
We show that if \BoostAttempt does not stop at step 2(e) of some iteration, 
then in every iteration $t$, the provided hypothesis $h_t$ is a $\mleft(\frac{1}{2}-\frac{1}{50}\mright)$-weak hypothesis 
with respect to the current distribution $p_t$ in the boosting process:
recall from the preliminaries that $p_t$ is a distribution on $S$, 
which is defined by the weight function~$W_t$, i.e.\ the probability of each example $z$ is proportional to~$W_t(Z)$. 
To establish the above we use two crucial properties of $h_t$:
\begin{itemize}
    \item The hypothesis $h_t$ satisfies
    \[
    L_{D_t}(h_t)\leq 1/100,
    \]
    where $D_t$ is the distribution defined in step 2(c), i.e.\ it is the mixture of the uniform distributions over the $S'_i$'s weighted by $\frac{W_t^{(i)}}{W_t}$. 
    \item $S'_i$ is a $\frac{1}{100}$-approximation of the distribution $p_t^i$ on $S_i$, defined by $p_t^i(z^i_j)=\frac{W_t(z^i_j)}{W_t^{(i)}}$ , and hence
    \[
    \mleft|L_{S'_i}(h) - L_{p_t^i}(h)\mright|
    \leq
    1/100
    \]
    for all $h\in \mathcal{H}$.
\end{itemize}

Let $p_t$ be the normalization of the weights in iteration $t$, that is $p_t(z^i_j)= \frac{W_t(z^i_j)}{W_t}$. So:

\begin{align*}
L_{p_t}(h_t) 
&=
\sum _{i=1}^k \sum_{z^i_j\in S_i}p_t(z^i_j)1[h_t(x^i_j)\neq y^i_j] \\
&=
\sum _{i=1}^k \sum_{z^i_j\in S_i}\frac{W_t(z^i_j)}{W_t}1[h_t(x^i_j)\neq y^i_j] \\
&=
\sum _{i=1}^k \frac{W_t^{(i)}}{W_t} \sum_{z^i_j\in S_i} \frac{W_t(z^i_j)}{W_t^{(i)}} 1[h_t(x^i_j)\neq y^i_j] \\
&=
\sum _{i=1}^k \frac{W_t^{(i)}}{W_t} L_{p_t^i}(h_t) \\
& \leq
\sum _{i=1}^k \frac{W_t^{(i)}}{W_t} \mleft[L_{S'_i}(h_t) +1/100\mright] \\
& =
\sum _{i=1}^k \frac{W_t^{(i)}}{W_t} \mleft[\frac{\sum_{\hat{z}^i_j\in S'_i}1[h_t(\hat{x}^i_j)\neq \hat{y}^i_j]}{|S'_i|} +1/100\mright] \\
& =
L_{D_t}(h_t) + 1/100 \\
& \leq
1/100+1/100 = 1/50.
\end{align*}

Since $1/50 < 1/15$, by Theorem~\ref{theo:boosting} a total of $\lceil 6\log \lvert S\rvert \rceil$ iterations are enough to output a classifier~$f$ that satisfies $E_S(f)=0$.
\end{proof}

Next, we we consider the case in which \BoostAttempt does get stuck. In this case, note that the small sample $S'$ sent to the center is not realizable.

\begin{observation}\label{obs:upp_acc_stuck}
Let $D$ be a distribution over a sample $S$. If for all $h\in \mathcal{H}$ it holds that $L_{D}(h)> 1/100$ then $S$ is not realizable. 
\end{observation}

The following observation states that \BoostAttempt is called at most $\OPT$ times by \AccuratelyClassify.

\begin{observation} \label{obs:upp_acc}
Let $S$ be a non-realizable sample, and let $S'$ be a non-realizable subsample of $S$. Then for all $h\in \mathcal{H}$,
\[
E_S(h)>E_{S\backslash S'}(h).
\]
That is, if we remove any non-realizable subsample from $S$, then the \emph{number} of mistakes of any hypothesis decreases by at least $1$.
\end{observation}


We are now ready to prove Theorem~\ref{theo:upp_acc}. The main part is analysing the communication complexity of \BoostAttempt. 

\begin{proof}[Theorem~\ref{theo:upp_acc}]
First we show correctness, and then analyze the communication complexity.
\paragraph{Correctness.} The loop in \AccuratelyClassify is executed as long as \BoostAttempt returns a non-realizable sample. Due to Observation~\ref{obs:upp_acc}, after at most $\OPT$ iterations, \BoostAttempt will return a classifier, since the input sample will then be realizable. 
This classifier makes zero errors on the input to \BoostAttempt, due to Lemma~\ref{lem:upp_acc_zero_err}. 
Consequently, the classifier $f$ returned by \AccuratelyClassify makes the least number of errors among all possible classifiers. Furthermore, if $S$ contains no contradicting examples, then $E_S(f)=0$.

\paragraph{Communication.} We first analyze the communication complexity of \BoostAttempt
    and show that its upper bounded by $O(k\log |S|(d \log n+ \log \lvert S \rvert))$.
    First, it has $\lceil 6\log\lvert S\rvert\rceil = O(\log \lvert S\rvert)$ iterations. 
    In each iteration, $k$ many $\frac{1}{100}$-approximations are sent to the center in step 2(a), each taking $O(d\log n)$ bits to encode, according to \citep*{vapnik71uniform}. 
    Then, the sums of weights of each player are sent to the center in step 2(b). This requires $O(k \log \lvert S \rvert)$ communication: 
    indeed, the initial weight of each element is $1$, and in each iteration it might be halved. 
    There are $O(\log \lvert S \rvert)$ iterations, so the weight of any element may decrease up to $\Omega(1/\lvert S \rvert)$. 
    So, encoding the sums of weights in step 2(b) requires $O(k\log \lvert S \rvert)$ bits. 
    Steps 2(c-e) can now be executed by the center, with zero communication. 
    Now, if the condition in step 2(d) does not hold, a non-realizable sample $S'$, which is the concatenation of the $\frac{1}{100}$-approximations $S'_i$,
    is outputted by \BoostAttempt. This step requires $k$ bit of communication, in which the center indicates to each of the players 
    that this condition does not hold. Also notice that this step happens at most once and hence increases the total communication complexity by at most $k$ bits.
    If this condition holds and the protocol continues, then each player updates its weights with zero communication.
    Thus, we get a total of $O(k\log |S|(d \log n+ \log \lvert S \rvert))$ communication used in \BoostAttempt. 

\AccuratelyClassify executes \BoostAttempt at most $\OPT$ times due to Observation~\ref{obs:upp_acc}, and hence the total communication used by \AccuratelyClassify is $O(\OPT \cdot k\log \lvert S \rvert(d \log n+ \log \lvert S \rvert))$.
\end{proof}

\paragraph{A computationally efficient implementation.}
We defined \BoostAttempt as a communication-efficient deterministic protocol. However, as currently formulated, the protocol is not computationally efficient, since step 2(a) requires finding a $\frac{1}{100}$-approximation, which cannot be done efficiently in general. \citet*{vapnik71uniform} proved that a random sample of size $O(d/\epsilon^2)$ is an $\epsilon$-approximation with high probability. 
This can be used to make our protocol efficient at the cost of making it randomized. Furthermore, notice that in step 2(d), a weak hypothesis for the distribution $D_t$ on $S'$ is found by the center. 
This step can also be implemented efficiently provided that $\mathcal{H}$ admits an efficient agnostic PAC learner in the centralized setting.


\section{A complementing negative result} \label{sec:lb-exact}

In this section we prove Theorem~\ref{theo:low_acc_opt_int}.

\begin{theorem*}[Theorem~\ref{theo:low_acc_opt_int} restatement] 
Let $\mathcal{H}=\{h_n: n\in \mathbb{N}\}$, where $h_n(i)=1$ if and only if $i=n$, be the class of singletons over $\mathbb{N}$.
If $T(n) = \log^{\omega(1)} n$ then $\mathcal{H}$ is not learnable under the promise that $\OPT\leq T(n)$, even when there are only $k=2$ players.
\end{theorem*}

The proof uses a mapping suggested in \citep*{kane2019communication} together with Theorem~\ref{theo:set_disj}, the well-known communication lower bound for set disjointness.

\begin{lemma} [\citet*{kane2019communication}] \label{lem:low_acc_map}
Let $x,y\in \{0,1\}^n$, and let $w(x)$ denote the hamming weight of a binary string $x$. Let $\mathcal{H}$ be the class of singletons over $[n]$ (it contains exactly all hypotheses that assign $1$ to a single $i\in [n]$ and $-1$ to all other elements). Then, there are mappings $F_a,F_b\colon \{0,1\}^n\rightarrow \mleft([n] \times \{\pm 1\} \mright)^n$ taking boolean $n$-vectors to samples such that the combined sample $S:=\langle F_a(x);F_b(y) \rangle$ satisfies:
\begin{enumerate}
    \item If $\DISJ_n(x,y)=1$ then $E_S(f) \geq  w(x)+w(y)$ for any classifier $f$ (not necessarily from $\mathcal{H}$).
    \item If $\DISJ_n(x,y)=0$ then the optimal $h\in \mathcal{H}$ satisfies $E_S(h) = w(x)+w(y)-2$.
\end{enumerate}
\end{lemma}
The proof follows by letting
\begin{align*}
F_a(x)=\mleft\{(i,(-1)^{1-x_i}): i\in [n]\mright\}, \\
F_b(y)=\mleft\{(i,(-1)^{1-y_i}): i\in [n]\mright\}.
\end{align*}
Those mappings are used in \citep*{kane2019communication} to prove a reduction to set disjointness, in order to show that agnostic classification requires $\Omega(n)$ communication under some conditions. A slight modification of their proof results in the bound of Theorem~\ref{theo:low_acc_opt_int}. 

\begin{proof}[Proof of Theorem~\ref{theo:low_acc_opt_int}.]
Let $n\in\mathbb{N}$ and set $U=[n]$. 
Given a randomized improper learning protocol $\pi(U)$ for $\mathcal{H}$ under the promise that $\OPT\leq T(n)$, we construct the following protocol $\pi'$ for $\DISJ_r$,
where $r=\lfloor \frac{T(n)}{2}\rfloor$.
\begin{enumerate}
    \item Let $x,y\in\{0,1\}^r$ denote the inputs for $\DISJ_r$.
    \item Publish $w(x),w(y)$.
    \item Extend $x,y$ to strings $x',y' \in \{0,1\}^n$ by adding $n-r$ zeroes to each.
    \item Construct $S:=\langle F_a(x');F_b(y') \rangle$ as described in  Lemma~\ref{lem:low_acc_map}.
    \item Execute $\pi(S)$ and let $f$ be the hypothesis it outputs.
    \item Output $1$ if and only if $E_S(f) \geq w(x)+w(y)$.
\end{enumerate}

Note that by construction, $\OPT$ is at most $2r \leq T(n)$ (because any singleton $h_i$ where $i\leq r$ has error at most $2r$ on $S$). So, $\OPT\leq T(n)$ and therefore Lemma~\ref{lem:low_acc_map} implies that this protocol solves set disjointness correctly with probability at least $2/3$. Thus, by Theorem~\ref{theo:set_disj}, its communication complexity is $\Omega(r) = \Omega(T(n))$.

We now wrap up the proof by showing that the communication complexity of $\pi$ is not in \[\poly(\log n,\log \lvert S\rvert=\log n, k=2) = \polylog (n).\]
Indeed, the communication complexity of $\pi'$ is at most $2\log r$ larger than that of $\pi$. Thus, also the communication complexity of $\pi$ is $\Omega(r)= \Omega(T(n))$,
and by assumption $T(n) = \log^{\omega(1)}n$.
%
%
\end{proof}

\section{Suggestions for future research} \label{sec:open}
\paragraph{Characterizing agnostic learning.}
Our main result can be viewed as an agnostic learning protocol whose communication complexity depends linearly on $\OPT$. There are concept classes in which such dependence is necessary, as shown by Theorem~\ref{theo:low_acc_opt_int}. It is also easy to see that there are classes for which this dependence can be avoided, for example finite classes. 
Is there a natural characterization of those classes which are learnable without any promise on $\OPT$? Are there infinite classes with this property?

\paragraph{The approximation factor in semi-agnostic learning.}
\citet*{balcan2012distributed} and \citet*{chen2016communication} give efficient semi-agnostic learners that approximate the error of a best hypothesis from the class up to a multiplicative factor of $c\geq 4$. A simple alteration of the constants in their proofs improves the approximation factor to $2+\alpha$ for every $\alpha > 0$ (at the cost of higher communication complexity which deteriorates as $\alpha\to 0$). Can the multiplicative factor be further improved, say to $c$ for some $c \leq 2$? 

\paragraph{Bounded communication complexity and generalization.}
It is interesting to further explore the relationship between the communication complexity and the generalization capacity of distributed learning protocols.

\section*{Acknowledgments}
We thank an anonymous ALT 2022 reviewer for pointing out that Theorem~\ref{theo:upp_acc_int} can be proved by reduction to semi-agnostic learning.


\bibliography{bib}
\bibliographystyle{icml2022}




\end{document}